\documentclass[a4paper,10pt]{article}
\usepackage{colt09e,times,dsfonts}
\usepackage{amsmath}
\usepackage{amsfonts}
\usepackage{amssymb}
\usepackage{subfigure}
\usepackage{epsfig}
\usepackage{color}
\usepackage{enumerate}
\usepackage[boxruled]{algorithm2e}
\newtheorem{remark}{Remark}

\newcommand{\X}{\mathcal{X}}
\newcommand{\Y}{\mathcal{Y}}
\newcommand{\Xspl}{\mathbf{X}}
\newcommand{\Yspl}{\mathbf{Y}}
\newcommand{\V}{\mathcal{V}}
\newcommand{\Sh}[2]{\mathcal{S}\left(#1, #2\right)}

\newcommand{\Half}{\mathcal{H}}
\newcommand{\C}{\mathcal{C}}
\newcommand{\B}{\mathcal{B}}
\newcommand{\Prj}{\mathcal{P}}

\newcommand{\A}{\mathbf{A}}
\newcommand{\Ar}{\A^{\!\!*}}

\newcommand{\ft}{\widetilde{f}}

\newcommand{\real}{\mathbb{R}}
\newcommand{\expectation}{\mathbb{E}}

\newcommand{\ex}[1]{{\expectation}\left[#1\right]}
\newcommand{\expec}{{\expectation} \,}
\newcommand{\expecf}[1]{{\expectation_{#1}} \,}
\newcommand{\norm}[1]{\left\|#1\right\|}
\newcommand{\abs}[1]{\left|#1\right|}
\newcommand{\paren}[1]{\left(#1\right)}

\newcommand{\Tr}{\mathbf{\top}}

\newcommand{\diameter}[1]{\Delta_n\left(#1\right)}
\newcommand{\diam}[1]{\Delta_n^2\left(#1\right)}

\newcommand{\diamX}{\Delta^2_{\X}}

\newcommand{\ind}[1]{\mathds{1}_{#1}}

\DeclareMathOperator*{\argmin}{argmin}
\DeclareMathOperator*{\level}{level}
\newcommand{\lev}[1]{\level\left(#1\right)}
\DeclareMathOperator*{\Expectation}{\expectation}

\title{Escaping the curse of dimensionality with a tree-based regressor}
\author{Samory Kpotufe\\ UCSD CSE\\
{\tt skpotufe@cs.ucsd.edu}}

\begin{document}
\maketitle

\begin{abstract}
We present the first tree-based regressor whose convergence rate depends only on the intrinsic dimension of the data, namely its Assouad dimension. The regressor uses the RPtree partitioning procedure, a simple randomized variant of $k$-$d$ trees. 
\end{abstract}

\section{Introduction}
Non-parametric learning algorithms tend to suffer from what is referred to as the curse of dimensionality, namely that prediction performance deteriorates dramatically as the number of features increases. This phenomenon is quantifiable in the case of regression algorithms: as initially shown by Stone \cite{S:60, S:61}, if we only assume that the regression function $f(x)$ is Lipschitz \footnote{Stone's result concerns a much larger class of regression functions; here we focus on Lipschitz conditions.} in $\real^D$, then no non-parametric estimator can achieve a convergence rate faster than $n^{-2/(2+D)}$. In other words, the number of points required to attain a low risk may be exponential in $D$, and this is infeasible even for moderate values of $D$.
\par
However, it is often the case that data which appears high dimensional, actually conforms to a structure of low intrinsic dimensionality (interpreted broadly). Examples of such situations are traditional continuous settings where the data is close to a low dimensional submanifold of $\real^D$, and discrete settings such as when the data is sparse. These are all examples of data with low Assouad dimension (see definition \ref{def:assouad}); this notion of dimension thus offers a natural and broad model of intrinsic data complexity.
\par
We show that, for any input data distribution, the risk of a regressor based on RPtree (a variant of $k$-$d$ tree) depends just on the unknown Assouad dimension of the data, regardless of the ambient dimension $D$. This is the first such result for tree-based regression.

\begin{figure}
\centering
\mbox{
\hspace{-0.02cm}
\subfigure[Dyadic tree] 
{
\includegraphics[height=2.2cm]{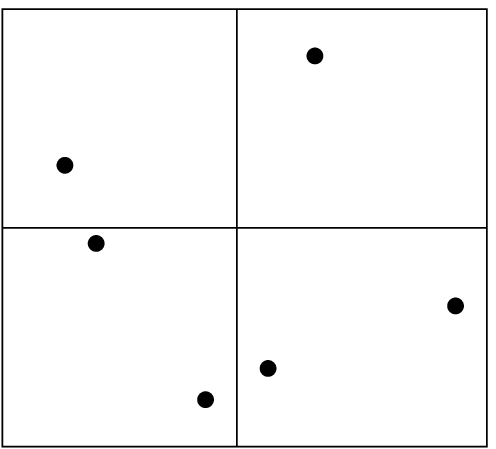}
\label{fig:1a}
}
\hspace{0.02cm}
\subfigure[$k$-$d$ tree] 
{
\includegraphics[height=2.2cm]{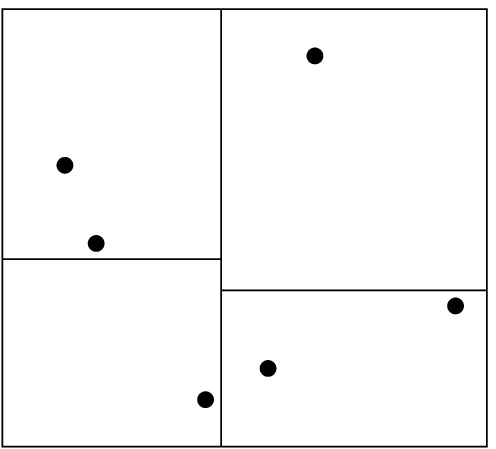}
\label{fig:1b}
}
\hspace{0.02cm}
\subfigure[RPtree] 
{
\includegraphics[height=2.2cm]{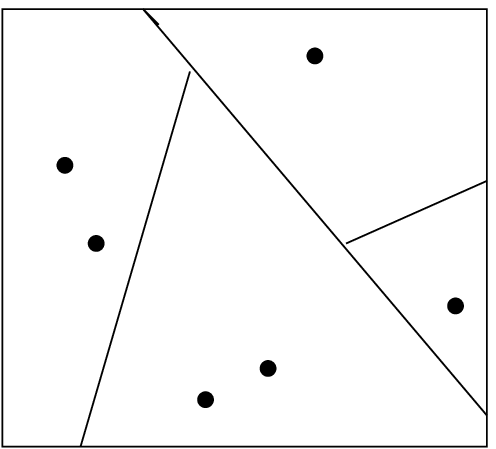}
\label{fig:1c}
}
}
\caption{Spatial partitioning induced by various splitting rules. Two levels or the tree are shown for each.}
\label{fig:1}
\end{figure}

\subsection{Tree-based regression}
Tree-based regression consists of first building a hierarchy of nested partitions of the data space (the tree), and then learning a piecewise continuous function $f_n$ over the cells of some chosen partition in the hierarchy. Future evaluations of $f_n(x)$ can be done in time just $O(\log n)$ by navigating the usually shallow tree down to an appropriate cell. These methods are popular due to their ease of use and computational efficiency (e.g. CART, dyadic trees, $k$-$d$ tree, see \cite{GN:67, SN:66, DGL:73}), but none has been shown to adapt to intrinsic dimensionality in terms of their regression risk. See figure \ref{fig:1} for some examples.
\par
The Random Projection tree (RPtree) is a hierarchical partitioning procedure which recursively bisects the data space with random hyperplanes (see figure \ref{fig:1c}). Although RPtree's connections to intrinsic data dimensionality has been studied in unsupervised settings (\cite{DF:59, GLZ:73}), its use for regression has not been explored. 
\par
Using RPtrees for regression requires a method for selecting a partition on which to learn the regressor $f_n$. Selecting a good partition from the hierarchy is essential to balancing the bias and variance of the regressor. Traditional methods use penalized empirical risk minimization over all possible partitions induced by the tree. Our approach can be more efficient in practice. We grow the tree in careful steps that enable us to quickly identify a small set of candidate partitions. We then provide a couple of options for selecting the final partition: one is to use cross-validation over the candidate partitions, another is a criterion which allows to automatically stop growing the tree when a good partition is attained. The latter method is computationally cheaper, while the former method results in a slightly better risk. In both cases, the excess risk of the RPTree regressor depends just on the unknown Assouad dimension of the input space, for all distributions.
\par
On the technical side, RPtree regression requires novel techniques for analyzing the bias of the estimator. Estimator bias is well understood to decrease with the diameters of the partition's cells. Unfortunately these \emph{physical diameters} are hard to assess for RPtrees given the random and irregular shapes of the cells, and in fact they may not decrease at all. However, we can track the diameters of the data within the cells, and we develop new techniques to relate these empirical \emph{data diameters} to the estimator's bias. We believe these techniques are of independent interest as they take focus away from the cells' physical diameters, thus opening the door to richer partitioning rules whose cell diameters are hard to control.


\subsection{Background and related work}
The realization that data is often less complex than indicated by the ambient dimension has spurred a significant body of work (referred to as manifold learning) that aim to embed the data into a low dimensional euclidean space (see e.g. \cite{RS:62, BN:63, TDL:64}). A possible approach to regression on high dimensional data
is to first reduce dimension using manifold learning and learn the regressor in the new space. Unfortunately, this approach is not guaranteed to work since pertinent information may be lost by the embedding. This raises the following natural question: can learning methods such as regression adapt automatically to data that has low intrinsic dimensionality while operating in the original space $\real^D$?
\par 
An important result in the direction of adaptive regression is the realization by Bickel and Li \cite{BL:65} that standard kernel regressors are adaptive in the following sense:
there exists an appropriate bandwidth setting such that the asymptotic pointwise risk at $x\in\real^D$ depends just on the manifold dimension and on the behavior of the kernel in a neighborhood of $x$. One then has to search for the appropriate bandwidth setting, either by estimating the manifold dimension or through cross validation over all possible values of this dimension (see e.g. \cite{BL:65, LW:68}).
\par
Kernel regressors can be expensive in practice: the kernel weights must be computed anew at each training point in order to evaluate the regressor on a new data point. This translates into an evaluation time of $\Omega(n)$ which is often a burden given large samples. Contrast this with the $O(\log n)$ evaluation time of tree-based regressors.
\par
In the case of classification, a recent result by Scott and Nowak (\cite{SN:66}) for dyadic decision trees is related: they show that if the input data is drawn from an approximately uniform measure on a manifold, and the Bayes decision boundary is sufficiently smooth, DDTs achieve classification rates that depend just on the manifold dimension. It is unclear whether their result will apply in a distribution free regression setting.

\section{Detailed overview of results}

We're given i.i.d training data $(\Xspl, \Yspl)=\{(X_i, Y_i)\}_{i= 1}^n$ $\in (\X\times\Y)^n$, where the input space $\X\subset\real^D$ is contained in a ball\footnote{We assume a Euclidean $l_2$ norm in this work.} of (unknown) diameter $\Delta_{\X}$, and the output space $\Y\subset \real^{D'}$ is contained in a ball of (unknown) diameter $\Delta_{\Y}$.

\subsection{Assouad dimension}
We model the intrinsic dimensionality of the space $\X$ using the notion of Assouad dimension defined below.
\par 
\begin{definition}
\label{def:assouad}
The Assouad dimension (or doubling dimension) of $\X\subset\real^D$ is the smallest $d$ such that for any ball $B\subset \real^D$, the set $B\cap \X$ can be covered by $2^d$ balls of half the radius of $B$.
\end{definition}

The Assouad dimension has proved useful in capturing the intrinsic complexity of data spaces as shown in various works on data analysis (see e.g. \cite{IN:70,BKL:71, C:74}).\\
It coincides with the natural notions of dimension of various geometric objects: it is easy to see that $d$-dimensional cubes, spheres, all have Assouad dimension $O(d)$ (see e.g. \cite{C:74}). It also captures notions of data complexity that are standard in the machine learning and statistics communities; this is stated in the following remarks for emphasis. 

\begin{remark}
 A $d$-dimensional hyperplane in $\real^D$ has Assouad dimension $O(d)$ (see \cite{C:74}).
\end{remark}

\begin{remark}
A $d$-dimensional Riemannian submanifold of $\real^D$ has Assouad dimension $O(d)$, subject to a bound on its curvature (see theorem 22 of \cite{DF:59}).
\end{remark}

\begin{remark}
A $d$-sparse data space in $\real^D$, i.e. one where each data point has at most $d$ non zero coordinates, has Assouad dimension $O(d\log D)$: it can be described by $\binom{D}{d}\leq D^d$ hyperplanes of dimension $d$.
\end{remark}

\begin{figure}
\centering
\mbox{
\hspace{-0.02cm}
\subfigure[Sparse data set.] 
{
\includegraphics[height=2cm, width = 2.5cm]{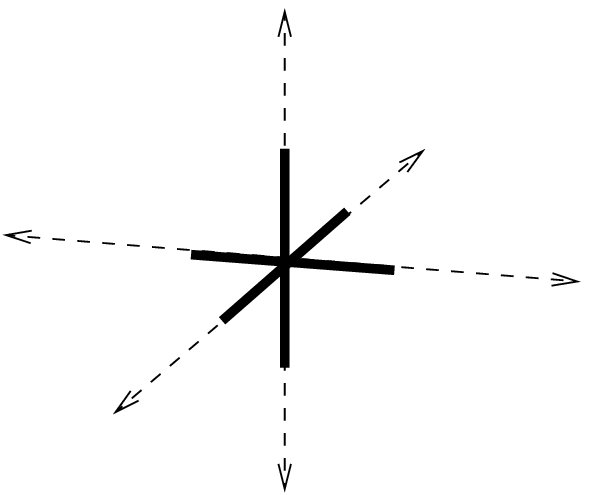}
\label{fig:sparseData}
}
\hspace{0.5cm}
\subfigure[2-$d$ manifold.] 
{
\includegraphics[height=2.3cm, width = 2.5 cm]{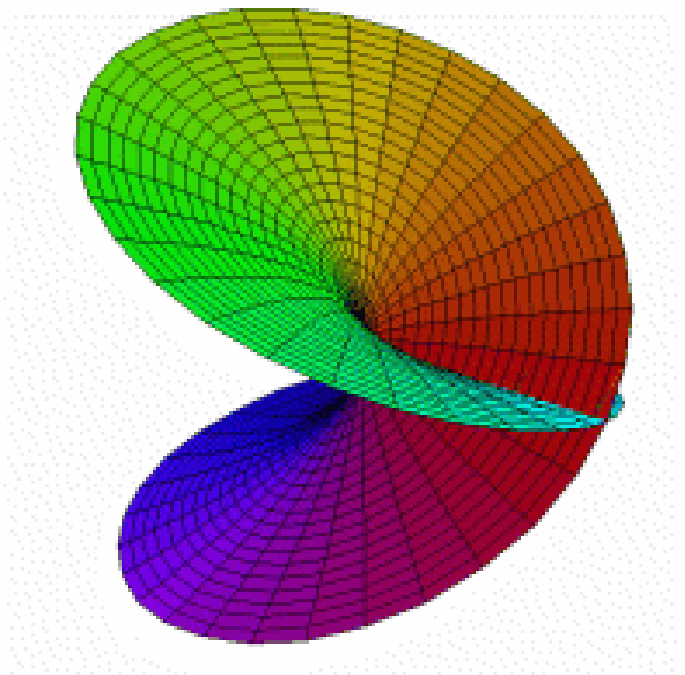}
\label{fig:manifold}
}
}
\caption{Examples of data with low Assouad dimension.}
\end{figure}

\subsection{Notions of diameter}
\label{sec:diameters}
Let $\A$ be some partition of $\X$. Traditionally, 
bias analysis revolves around the \emph{physical diameters} $\Delta(A) \doteq \displaystyle\max_{x, x' \in A} \norm{x-x'}$ of cells $A\in\A$ (see e.g. \cite{GN:67, SN:66, DGL:73}). In this work we instead relate bias to the \emph{data diameters} of the cells, that is $\Delta_n(A) \doteq \displaystyle\max_{x, x' \in A\cap \Xspl} \norm{x-x'}$ or $0$ if $A\cap\Xspl = \emptyset$. 
\begin{figure}[h]
 \centering
\resizebox{3.5cm}{!}{\input{Diameters.pstex_t}}
\end{figure}
\par 
Focusing on data diameter has the following advantage. We never need to evaluate the physical diameters of the cells, and these need not decrease. Consequently, we don't have to constrain the partition to regular shaped cells (e.g. axis parallel hyper-rectangles) whose physical diameters are easily controlled. In particular, it opens the door to richer partitioning rules such as RPtree which adapt better to the data complexity at the expense of creating irregular cells. We expand on this last point in the example below.
\par 
Consider a data space of the following form:
 $$\cup_{i\neq j}\{te_i \,\pm\, \varepsilon e_j: t\in [-1, 1]\},\,i, j\in [D], \text{for a fixed $\varepsilon <\!\!< 1$}.$$
This is an extreme case of a noisy sparse data set of Assouad dimension $O(\log D)$, depicted in figure \ref{fig:sparseData}. We'd like to partition this space in a way that reduces the data diameters of the cells (for low estimator bias) while achieving a small partition size (for low estimator variance). 
Axis parallel splitting rules such as $k$-$d$ trees or dyadic trees would require a number of cells exponential in $D$ in order to halve the diameters. Yet, the set itself can be partitioned into at most $2D^2$ cells of half its radius. The richness of random splits allows us to achieve a partitioning just a bit larger than this, even in the worst case over distributions on the set. In fact, given any data set of Assouad dimension $d$, RPtrees are guaranteed to achieve a partition of size at most $2^{\widetilde{O}(d)}$, such that the data diameters of each cell is at most half of the diameter of the full data set. We refer the reader to \cite{DF:59} for a detailed analysis.
\par 
We'll soon see that, for low estimator bias, we don't need every cell of a partition to have small data diameter, but rather that these diameters are small in an average sense. Given a collection $\A$ of disjoint subsets of $\X$, we define the following notion of average data diameter:
\begin{eqnarray*}
\Delta_n(\A) &\doteq& \paren{\frac{\sum_{A\in \A}\mu_n(A) \diam{A}}{\sum_{A\in\A}\mu_n(A)}}^{1/2}, 
\end{eqnarray*}
where $\mu_n$ is the empirical measure over $\Xspl$ (we'll let $\mu$ denote the marginal measure over $\X$).

\subsection{Regression setup}
We assume that the regression function $f(x) = \ex{Y|X = x}$ is $\lambda$-Lipschitz, for an unknown parameter $\lambda$:
$$\forall x, x' \in \X, \, \norm{f(x) - f(x')}\leq \lambda \norm{x-x'}.$$
\par 
For any function $g(x): \X\mapsto \Y$, the $l_2$ pointwise risk at $x$ satisfies
\begin{eqnarray*}
R(g(x)) \doteq \expecf{Y} \norm{Y-g(x)}^2 = R(f(x)) + \norm{f(x) - g(x)}^2, 
\end{eqnarray*}
and the integrated risk can then be written as
\begin{eqnarray*}
R(g) \doteq\expecf{X}R(g(X)) = R(f) + \expecf{X}\norm{f(X) - g(X)}^2.
\end{eqnarray*}
Thus, the pointwise excess risk of $g(x)$ over $f(x)$ is simply $\norm{f(x) - g(x)}^2$. In this paper we'll be interested in the integrated excess risk 
\begin{equation*}
 \norm{f -g}^2 \doteq R(g) - R(f) = \expecf{X}\norm{f(X) - g(X)}^2.
\end{equation*} 

\subsection{Choosing a good partition for regression}
\begin{procedure*}[t]
$\A^0\leftarrow \{\X\}$\;
\For{$i\leftarrow 1$ \KwTo $\infty$}{
  \ForEach{cell $A\in\A^{i-1}$} {
     \CommentSty{// Create a subtree rooted at $A$:}\\
$l \leftarrow \lev{A}$ in the current tree \tcp*[l]{Root is at level 0}
     (subtree rooted at $A$) $\leftarrow \mathtt{coreRPtree}\paren{A,\,\Delta_n(A)/2,\delta, l}$\; 
  }
$\A^i\leftarrow$ partition of $\X$ defined by the leaves of the current tree\;
$\lev{\A^i}\leftarrow \max_{A\in \A^i}\lev{A}$ \;
\vspace{1mm}
\CommentSty{// At this point we have two options for stopping and returning a partition.}\\
\vspace{1mm}

\CommentSty{\fbox{Option 1: Cross-validation}}\\
\vspace{0.5mm}
\If {$\Delta_n\paren{\A^i} = 0$ or $\lev{\A^i}\geq \log n^2$} {
Draw test sample $(\Xspl', \Yspl')$ of size $n$ and define $R_n'(\cdot)$ as the empirical risk over the test sample\;
$\, \Ar\leftarrow \displaystyle\argmin_{\A^j\in\{\A^0, \ldots, \A^i\}} R_n'(f_{n,\A^j})$\;
\vspace{0.5mm}
\Return $f_n\doteq f_{n,{\Ar}}$\;
 }
\CommentSty{\fbox{Option 2: Automatic stopping}}\\
\vspace{0.5mm}
$\alpha(n) \leftarrow \paren{\log^2 n}\log\log(n/\delta)+ \log (1/\delta)$\;
\vspace{0.5mm}
  \If {$\lev{\A^i}\geq \log \left({n\cdot\diam{\A^i}}/{\alpha(n)\diam{\X}}\right)$} {
\vspace{0.5mm}
$\,\Ar\leftarrow \displaystyle\argmin_{\A^j \in \left\{\A^{i-1},\, \A^i\right\}} \paren{\frac{\alpha(n)}{n}\cdot \abs{\A^j} + \diam{\A^j}}$\;
\vspace{0.5mm}
\Return $f_n \doteq f_{n,\Ar}$\;
  }
}
\caption{adaptiveRPtree(sample $\Xspl$, confidence parameter $\delta$)\label{proc:1}}
\end{procedure*}

A tree-based regressor works in two phases. The partitioning phase returns a partition $\A$ of the data space $\X$ and a final regressor is learned as a piecewise continuous function over the cells of $\A$. In this work we'll consider a piecewise constant regressor over the returned partition $\A$ defined as follows:
\par 
\noindent For $x\in \X$, let $\A(x)$ be the cell of $\A$ to which $x$ belongs. If $\mu_n(\A(x)) > 0$, the regressor is obtained as
\begin{equation*}
f_{n, \A} (x) \doteq 
\frac{\sum_{i= 1}^n Y_i \cdot \ind{X_i\in \A(x)} }{n\cdot\mu_n(\A(x))},
\end{equation*}
otherwise use a default setting $f_{n, \A}(x) = y_0 \in \Y$ whenever $\A(x)$ is empty of training points. We'll often refer to the final regressor as $f_n(\cdot)$ as long as the partition used for the estimate is clear from context.
\par 
Procedure $\mathtt{adaptiveRPtree}$ makes calls to the the subprocedure $\mathtt{coreRPtree}$ which implements the basic RPtree splits. We defer the complete treatment of this subprocedure to section \ref{sec:coreRPtree} since most of the analysis will concern $\mathtt{adaptiveRPtree}$. For now, note that the call to $\mathtt{coreRPtree}$ returns a subtree rooted at $A$ with the following property: let $\A$ be the collection of subsets of $\real^D$ defined by the leaves of this subtree, we have $\Delta_n(\A) \leq \Delta_n(A)/2$. Also, the implementation of $\mathtt{coreRPtree}$ ensures that the final tree built by $\mathtt{adaptiveRPtree}$ has height at most $6\log n$. 
\par
Procedure $\mathtt{adaptiveRPtree}$ grows the tree in steps $\A^0$, $\A^1, \ldots$, where 
$\Delta_n\paren{\A^{i+1}}\leq\Delta_n\paren{\A^{i}}/2$, and eventually returns one of the partitions $\A^i$ for some $i$. We present a couple of options for selecting a good partition to return. The first option uses cross-validation: grow a large tree and prune it back by minimizing empirical risks over an i.i.d test sample $(\Xspl', \Yspl')$ of size $n$. The other option is that of automatic stopping: we return a partition as soon as some stopping condition is met.
\par
The two options for selecting the return partition are outlined in procedure $\mathtt{adaptiveRPtree}$. The empirical risk in the cross-validation option is defined as 
\begin{equation*}
 R_n'(g) \doteq \frac{1}{n}\sum_{i\in[n]}\norm{Y_i' - g(X_i')}^2.
\end{equation*}
\par
The automatic stopping option returns one of two partitions and requires no test sample. It is a computationally faster option and, as we'll see, the resulting bounds are only marginally worsened.

\subsection{Main Results}

\begin{definition}
\label{def:decRate}
Given a sample $\Xspl$, we say that $\mathtt{adaptiveRPtree}$ attains a diameter decrease rate of $k$ on $\Xspl$ for $k \geq d$, if every call to the subprocedure $\mathtt{coreRPtree}\paren{A,\diameter{A}/2, \delta, l}$ in the second loop of the procedure returns a tree rooted at $A$ of depth at most $k$.
\end{definition}

\begin{theorem}
Assume that $\X$ has Assouad dimension $d$. There exist constants $C, \,C'$ independent of $d$ and $\mu(\X)$, such that the following holds. 
\par
Suppose the cross-validation option is used. Define $$\alpha(n)\doteq\paren{\log^2 n}\log\log(n/\delta)+ \log (1/\delta),$$ and assume $n\geq \max \left\{\paren{{\lambda\Delta_\X}/{\Delta_\Y}}^2, \alpha(n)\right\}.$
With probability at least $1-\delta$, the algorithm attains a 
diameter decrease rate of $k\leq C'd\log d$, and the excess risk of the regressor satisfies
\begin{eqnarray*}
 \norm{f_n - f}^2 &\leq& C\cdot\paren{\lambda\Delta_\X}^{2k/(2+k)}\paren{\frac{\Delta_{\Y}^2\cdot\alpha(n)}{n}}^{2/(2+k)} \\
&\,& + 2\Delta_\Y^2\sqrt{\frac{\ln\log n^6 + \ln 3/\delta}{2n}}. 
\end{eqnarray*}
\label{theo:main1}
\end{theorem}

\begin{theorem}
Assume that $\X$ has Assouad dimension $d$. There exist constants $C, \,C'$ independent of $d$ and $\mu(\X)$, such that the following holds.  
\par
Suppose the automatic stopping option is used. Define $$\alpha(n)\doteq\paren{\log^2 n}\log\log(n/\delta)+ \log (1/\delta).$$ With probability at least $1-\delta$, the algorithm attains a 
diameter decrease rate of $k\leq C'd\log d$, and the excess risk of the regressor satisfies
\begin{eqnarray*}
 \norm{f_n - f}^2 \leq C\cdot\paren{\Delta_{\Y}^2 + \lambda^2}(\diamX +1)\cdot\paren{\frac{\alpha(n)}{n}}^{2/(2+k)}.
\end{eqnarray*}
\label{theo:main2}
\end{theorem}

\subsection*{Analysis outline}
We start in section \ref{sec:preliminaries} by laying out the necessary tools for the rest of the analysis.
\par
The theorems are then proved in two parts. First we bound the excess risk of the algorithm in terms of the observed diameter decrease rates in section \ref{sec:risk} (lemma \ref{lem:riskCVoption} for the cross -validation option, and lemma  \ref{lem:riskASoption} for the automatic stopping). We subsequently argue that these decrease rates depend just on the intrinsic dimensionality of the data (corollary \ref{cor:diameter-decrease-max} of section \ref{sec:decreaserates}).
\par 
Theorem \ref{theo:main1} results from lemma \ref{lem:riskCVoption} and corollary \ref{cor:diameter-decrease-max}, while theorem \ref{theo:main2} results from lemma \ref{lem:riskASoption} and corollary \ref{cor:diameter-decrease-max}.

\section{Proof preliminaries: risk bound for $f_{n, \A}$}
\label{sec:preliminaries}
In this section we develop the necessary tools to bound the excess risk of $f_{n, \A}$, where $\A$ is an RPtree partition, i.e. $\A$ is defined by the leaves of some subtree of the tree returned by $\mathtt{adaptiveRPtree}$.
\subsection{Generic decomposition of excess risk}
We start the analysis with a standard decomposition of the excess risk into bias and variance terms. Let $\A$ be any partition of $\X$. 
The following function of $x\in \X$ provides a bridge between the regressor $f_{n,\A}$ and the regression function $f$:
\begin{eqnarray*}
 \ft_{n,\A}(x) &\doteq& \displaystyle \Expectation_{\Yspl|\Xspl} f_{n,\A}(x) 
= \frac{\sum_{i= 1}^n f(X_i) \ind{X_i \in \A(x)}}{n\mu_n(\A(x))},  
\end{eqnarray*}
if $\mu_n(\A(x))\neq 0$, otherwise we set $\ft_{n,\A}(x) = y_{0} \in \Y$ .
\par
The pointwise excess risk can be bounded as  
\begin{eqnarray}
 \norm{f_{n,\A}(x) - f(x)}^2 &\leq& 2\norm{f_{n,\A}(x) - \ft_{n,\A}(x)}^2 \nonumber\\
&\,& + 2\norm{\ft_{n,\A}(x) - f(x)}^2.\label{eq:factor}
\end{eqnarray}
We therefore proceed by bounding each term on the r.h.s separately in the following two lemmas.

\begin{lemma}[Variance]
Let $\A$ be a partition of $\X$. The following inequality holds for all $x\in \X \,\text{s.t.} \,\mu_n(\A(x))>0$, with probability at least $1-\delta'$ over the random choice of $\Yspl$ for $\Xspl$ fixed: 
\begin{equation}
 \norm{f_{n,\A}(x) - \ft_{n,\A}(x)}^2 \leq \Delta_{\Y}^2\cdot\frac{2 + \ln(\abs{\A}/\delta')}{n \mu_n(\A(x))}.
\end{equation}
\label{lem:var}
\end{lemma}
\begin{proof}
Fix $\Xspl$. Now fix $A\in \A$, and let $x\in A$. We'll consider 
$\Yspl_{A} \doteq \{Y_i \in \Yspl\, \text{s.t.}\, X_i \in A\}$.
Write:
\begin{eqnarray*}
 \psi(\Yspl_{A}) &\doteq& \norm{f_{n,\A}(x) - \ft_{n,\A}(x)} \\ 
&=& \norm{\frac{\sum_{i= 1}^n (Y_i - f(X_i)) \ind{X_i \in A}}{n\mu_n(A)}}.
\end{eqnarray*}
We can now apply McDiarmid's inequality to $\psi(\cdot)$, as it is easy to verify that, changing one of the $Y$ values in $\Yspl_{A}$ changes the value of 
$\psi(\cdot)$ by at most $\frac{\Delta_{\Y}}{n\mu_n( A)}$. We then have that, 
\begin{equation*}
 \psi(\Yspl_{A}) \leq \expec{\psi(\Yspl_{A})} +  \Delta_{\Y}\cdot\sqrt{\frac{\ln(\abs{\A}/\delta')}{2n\mu_n(A)}}
\end{equation*}
with probability at least $1-\delta'/\abs{\A}$ over the random choice of $\Yspl_{A}$.
\par
The expectation can be bounded as follows
\begin{eqnarray*}
\expec{\psi(\Yspl_{A})} &\leq& \paren{\expec{\paren{\psi(\Yspl_{A})}^2}}^{1/2}\\
&=&  \paren{\expec\norm{\frac{\sum_{i= 1}^n (Y_i - f(X_i)) \ind{X_i \in A}}{n\mu_n( A)}}^2}^{1/2}\\
&\leq& \paren{\frac{\sum_{i= 1}^n \expec\norm{Y_i - f(X_i)}^2 \ind{X_i \in A}}
{\paren{n\mu_n( A)}^2}}^{1/2}\\
&\leq& \paren{\frac{\sum_{i= 1}^n \Delta_{\Y}^2 \ind{X_i \in A}}
{\paren{n\mu_n( A)}^2}}^{1/2} \\
&=& \frac{\Delta_{\Y}}{\sqrt{n\mu_n( A)}}.
\end{eqnarray*}
The first inequality above is an application of Jensen's inequality. The second inequality results from the fact that, for independent random vectors $v_i$ with null expectation, we have $\expec{\norm{\sum_i v_i}^2} = \sum_i \expec{\norm{v_i}^2}$; here we just take $v_i$ to be $(Y_i - f(X_i)) \ind{X_i \in A}/\paren{n\mu_n( A)}$.
\par
Combining the above yields the desired bound on $\psi(\Yspl_{A})$ with probability at least $1-\delta'/\abs{\A}$. We then conclude with a union bound over all $A\in \A$.
\end{proof}

\begin{lemma}[Bias]
Let $\A$ be a partition of $\X$. The following inequality holds for all $x\in \X \,\text{s.t.} \,\mu_n(\A(x))>0 $:
\begin{equation}
 \norm{\ft_{n,\A}(x) - f(x)}^2 \leq \lambda^2\Delta^2\paren{\A(x)}.
\end{equation}
\label{lem:bias}
\end{lemma}
\begin{proof}
Fix $A\in\A$ and let $x\in A$. Now write
 \begin{eqnarray*}
\norm{\ft_{n,\A}(x) - f(x)}^2  
&=& \norm{\frac{\sum_{i= 1}^n (f(X_i) - f(x)) \ind{X_i \in A}}{n\mu_n(A)}}^2\\
&\leq&  \paren{\frac{\sum_{i= 1}^n \norm{f(X_i) - f(x)} \ind{X_i \in A}}{n\mu_n(A)}}^2\\
&\leq& \paren{\frac{\sum_{i= 1}^n \lambda\norm{X_i - x} \ind{X_i \in A}}{n\mu_n(A)}}^2\\
&\leq& \lambda^2\Delta^2\paren{A},
\end{eqnarray*}
where the second inequality results from the Lipschitz condition on $f(\cdot)$.
\end{proof}
\par
In lemma \ref{lem:bias} above, the bias is bounded in terms of the \emph{physical diameters} $\Delta(A)$. However, for an RPtree partition $\A$ (i.e. $\A$ is defined by the leaves of some subtree), the physical diameters $\left\{\Delta(A), A\in\A\right\}$ could be as large as $\Delta_\X$, the diameter of the whole space. As previously discussed, RPtree focuses on decreasing the \emph{data diameters} $\Delta_n(A)$, and we'll argue that this is sufficient to decrease the bias of the estimator. For this purpose, we will replace RPtree partitions $\A$ with alternate partitions $\A'$ as explained in the next section.

\subsection{Alternate partitions}
\begin{figure}
\centering
\hspace{-0.02cm}
 \subfigure[Cover $\B$]{
\includegraphics[height=2.2cm,width = 2.2cm]{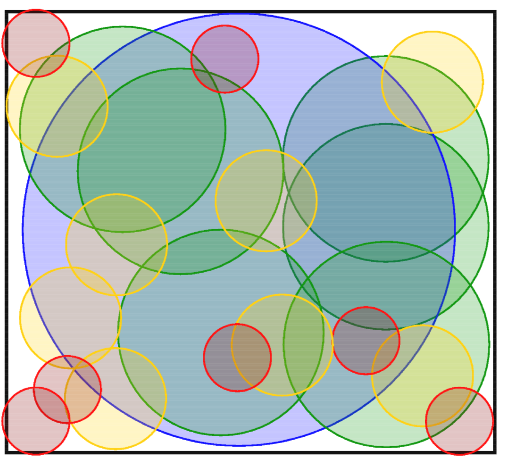}
}
\hspace{0.02cm}
\subfigure[Partition $\A$]{
\includegraphics[height=2.2cm,width = 2.2cm]{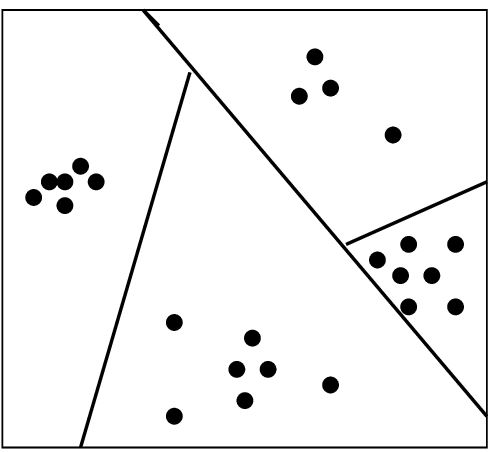}
}
\hspace{0.02cm}
\subfigure[Partition $\A'$]{
\includegraphics[height=2.2cm,width = 2.2cm]{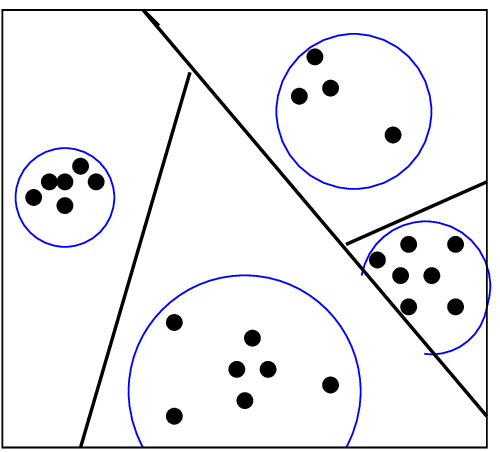}
}
\caption{We start with a cover $\B$ of $\X$ with balls of different size, next we see the data and obtain a partition $\A$, we then substitute $\A$ with $\A'$ by intersecting the cells of $\A$ with balls of $\B$.}
\label{fig:altpart}
\end{figure}

Given a partition $\A$ built by RPtree, we will consider an alternate partition $\A'$ which will serve to analyze the bias of the regressor $f_{n,\A}$ (see above discussion of lemma \ref{lem:bias}). Each cell of $\A'$ will either contain no data point, or has physical diameter roughly the same as its data diameter. This is done by intersecting the cells of $\A$ with balls or complements of balls from a fixed collection $\B$ defined below (see figure \ref{fig:altpart}). We'll see that $\A'$ approximately maintains key properties of $\A$, namely partition size and average data diameters.

\begin{definition}
We define $\B$ as the following collection of balls in $\real^D$. Let $I=\lfloor \log n^{2/(2+d)}  \rfloor$. For each $i=0$ to $ I$, consider a minimal $\paren{2^{-i}\Delta_\X}$-cover of $\X$; let $\B_i$ be the set of all balls $B\paren{z, 2^{-(i-2)}\Delta_\X}$ centered at points $z$ in the cover. We set 
$\B\doteq \cup_{i=0}^I\B_i$. 
\end{definition}
Every cell $A\in \A$ such that $A\cap\Xspl\neq \emptyset$
will be replaced in $\A'$ by two cells $A_1', A_2'$ obtained as follows.
\par 
Consider the smallest $i\in\{0, \ldots, I\}$ such that 
$2^{-i}\Delta_{\X} \leq \max \left\{\diameter{A}, 2^{-I}\Delta_{\X}\right\}$, i.e. $i = \min\left\{I, \lceil\log\frac{\Delta_\X}{\Delta_n(A)}\rceil\right\}$. There exists a ball $B\in \B_i$ which covers $A\cap\Xspl$: pick any $x\in A\cap \Xspl$, and pick the ball $B$ in $\B_i$ whose center $z$ is closest to $x$; we have $\forall x'\in A\cap\Xspl$, that $x'\in B = B\paren{z, 2^{-(i-2)}\Delta_\X}$ since by a triangle inequality
\begin{eqnarray*}
 \norm{z-x'}&\leq& \norm{z-x} + \norm{x-x'}\leq2^{-i}\Delta_\X + \diameter{A}\\
&\leq& 2^{-i}\Delta_\X + 2^{-(i-1)}\Delta_\X 
\leq 2^{-(i-2)}\Delta_\X.
\end{eqnarray*}

We define $A_1' = B\cap A$ and $A_2' = A\setminus A_1'$ for all $A\in \A, \,A\cap\Xspl\neq \emptyset$; on the other hand we let $A_1' = A$, $A_2'= \emptyset$ for all $A\in \A, \,A\cap\Xspl = \emptyset$. We finally define $\A'$ to be the collection of all such $A_1', A_2'$ over $A\in\A$.

In the following lemma we relate diameters of cells of $\A'$ to the data diameters of cells of $\A$.

\begin{lemma}[Diameters of $\A'$]
Let $\A$ be some partition of $\X$ and let $\A'$ as defined above. We have that
\begin{equation*}
\sum_{A'\in \A'}\mu_n(A')\Delta^2(A')\leq 64\diam{\A} + 256n^{-4/(2+d)}\cdot\diamX.
\end{equation*}
\label{lem:properties_A'}
\end{lemma}
\begin{proof}
Let $A\in \A, A\cap\Xspl\neq \emptyset$. We have $\mu_n(A_1') = \mu_n(A)$ and $\mu_n(A_2') = 0$. Also, given the smallest $i\in\{0, \ldots, I\}$ such that 
$2^{-i}\Delta_{\X} \leq \max \left\{\diameter{A}, 2^{-I}\Delta_{\X}\right\}$, we have that
\begin{itemize}
 \item $\diameter{A} > 2^{-I}\Delta_{\X}$ implies
 $\Delta(A_1')\leq 2\cdot2^{-(i-2)}\Delta_{\X}\leq 8\diameter{A},$
 \item $\diameter{A}\leq 2^{-I}\Delta_{\X}$ implies 
 $\Delta(A_1')\leq 2\cdot2^{-(I-2)}\Delta_{\X}\leq 16n^{-2/(2+d)}\cdot\Delta_{\X}$.
\end{itemize}
Therefore, let $\A_+ = \{A\in \A, \diameter{A} > 2^{-I}\Delta_{\X}\}$, we have 
\begin{eqnarray*}
\sum_{A'\in \A'}\mu_n(A')\Delta^2(A')
&=&  \sum_{A\in \A_+}\mu_n(A)\Delta^2(A_1') \\
&+&\sum_{A\in \A\setminus \A_+}\mu_n(A)\Delta^2(A_1')\\
&\leq& \sum_{A\in \A_+}64\mu_n(A)\diam{A} \\
&+& \sum_{A\in \A\setminus \A_+}256\mu_n(A)n^{-\frac{4}{2+d}}\cdot\diamX\\
&\leq& 64\diam{\A} + 256n^{-\frac{4}{2+d}}\cdot\diamX.
\end{eqnarray*}
\end{proof}

In order to bound the integrated excess risk, we'll need the empirical mass of cells of $\A'$ to be close to their true mass. In particular, this will allow us to effectively discard cells that are empty of data since they will have little effect on the integrated excess risk. The following lemma from VC theory will come in handy.

\begin{lemma}[Relative VC bounds -\cite{VC:72}]
Let $\C$ be a class of subsets of $\real^D$, and let its $2n$-shatter coefficient be given by $\Sh{\C}{2n}$.
With probability at least $1-\delta'$ over the choice of $\Xspl$, all $A'\in \C$ satisfy  
\begin{eqnarray}
\mu(A') &\leq& \mu_n(A') + 2\sqrt{\mu_n(A')\frac{\ln \Sh{\C}{2n}  + 
\ln ({4}/{\delta'})}{n}} \nonumber \\
&+& 4\frac{\ln \Sh{\C}{2n} + \ln ({4}/{\delta'})}{n}. \label{eq:vc} 
\end{eqnarray}
\label{lem:relativeVC}
\end{lemma}

The next lemma establishes the convergence of empirical masses of cells of $\A'$.

\begin{lemma}[Mass of cells of $\A'$]
With probability at least $1-\delta'$ over $\Xspl$ and the randomness in the algorithm, we have for all RPtree partitions $\A$, for all $A'\in\A'$ that
\begin{eqnarray*}
\mu(A') &\leq& \mu_n(A') + 2\sqrt{\mu_n(A')\frac{\V + 
\ln({4}/{\delta'})}{n}} \\
&+& 4\frac{\V + \ln ({4}/{\delta'})}{n}, \text{ where }\label{eq:vc2}\\
\V &\leq& O(\log n)(\log n + \log\!\log(1/\delta)).
\end{eqnarray*}

\label{lem:Shatter} 
\end{lemma}
\begin{proof}
Suppose w.l.o.g that the RPtree is built by picking random directions from a fixed collection $\Prj$ without replacement. How big should $\Prj$ be so we have enough directions to choose from? The implementation of $\mathtt{coreRPtree}$ ensures that $\abs{\Prj}\leq 2n^6\log\paren{6n^2/\delta}$ is sufficient (see remark \ref{rem:treeSize} of section \ref{sec:coreRPtree}).
Now fix such a collection $\Prj$ and let $\Half_\Prj$ be the union of $\{\X\}$ and the class of half spaces of $\real^D$ defined by hyperplanes normal to the directions in $\Prj$.
For an RPtree partition $\A$, each cell of $\A$ is the intersection of at most $6\log n$ elements of $\Half_\Prj$ since the tree is guaranteed to have height at most $6\log n$ (remark \ref{rem:treeSize}). Each cell of $\A'$ is the intersection of a ball or the complement of a ball in $\B$ with a cell of $\A$.
\par
All such cells therefore belong to the following class of subsets of $\real^D$:
$$\C = \left\{h: h = h_0\cap\paren{\displaystyle\bigcap_{l=1}^{6\log n} h_l}, \, h_0 \text{ or } h_0^\mathsf{C} \text{ is in }\B, h_l \in \Half_\Prj\right\}.$$

We now proceed to bounding $\Sh{\C}{2n}$, the $2n$-shatter coefficient of $\C$ as follows.
\par 
Given $2n$ sample points, every direction $v\in \Prj$ defines at most $2(2n+1)$ equivalent choices of half-spaces in $\real^D$. We therefore have 
\begin{eqnarray*}
\Sh{\C}{2n} &\leq& 2\abs{\B}\paren{(4n+2)\abs{\Prj} +1}^{6\log n} \\
&\leq& 2\abs{\B}\paren{n^6(8n+4)\log \paren{6n^2/\delta} +1}^{6\log n}.
\end{eqnarray*}
Since $\X$ has Assouad dimension $d$, we have $\abs{\B} \leq \sum_{i=0}^I 2^{di}\leq 2n^{2d/(2+d)}$.
The proof is completed by letting $\V = \log \Sh{\C}{2n}$ for $\Prj$ fixed, and calling on lemma \ref{lem:relativeVC}.
\end{proof}

\begin{lemma}[Excess risk]
There exists a constant $C_1$ independent of $d$ and $\mu(\X)$ such that the following holds with probability at least $1-\delta/3$ over the choice of $(\Xspl, \Yspl)$ and the randomness in the algorithm.
\par
Define $\alpha(n) \doteq \paren{\log^2n}\log\!\log (1/\delta) + \log (1/\delta)$. Let $\A^i$ be the final partition reached by $\mathtt{adaptiveRPtree}$. For all partitions $\A \in \left\{\A^j\right\}_{j = 0}^i$, we have
\begin{eqnarray*}
 \norm{f_{n,\A} - f}^2 &\leq& C_1\left(\Delta_{\Y}^2\abs{\A}
\frac{\alpha(n)}{n}\right.\\
&\,& \left. + \lambda^2\paren{\diam{\A} + n^{-4/(2+d)}\diamX}\vphantom{\frac{\alpha(n)}{n}}\right).
\end{eqnarray*}
\label{lem:risk}
\end{lemma}

\begin{proof}
Let the partition $\A \in \left\{\A^j\right\}_{j = 0}^i$ and the sample $\Xspl$ be fixed.
By lemma \ref{lem:Shatter} we have, with probability at least $1 - \delta'$, that equation (\ref{eq:vc2}) holds for all $A'\in \A'$ with 
$\V\leq O(\log n)(\log n + \log\!\log(1/\delta))$. 
\par
The excess risk decomposes over $\A'$ as 
\begin{equation*}
 \norm{f_{n, \A} - f}^2 = \sum_{A'\in \A'}\int_{A'} \norm{f_{n, \A}(x) - f(x)}^2 \mu(dx).
\end{equation*}
We next divide the cells of $\A'$ into two groups:
\begin{equation*}
\A_>' \doteq \left\{A' \in \A', \mu_n(A')\geq \frac{\V + \ln (4/\delta')}{n}\right\}, 
\end{equation*}
and $\A_<' \doteq \A' \setminus \A_>'$.
\par 
It's easy to see that from equation (\ref{eq:vc}), we have $\forall A'\in \A_>', \,\mu(A') \leq 7\mu_n(A')$, and $\forall A'\in \A_<', \,\mu(A') \leq 7\frac{\V + \ln (4/\delta')}{n}$.
\par 
Integrating over $\A'_<$, we have
\begin{eqnarray}
&\displaystyle\sum_{A'\in \A_<'}&\int_{A'} \norm{f_{n, \A}(x) - f(x)}^2 \mu(dx) \nonumber\\
&\leq& \sum_{A'\in \A_<'} \Delta_{\Y}^2 \cdot\mu(A')  \nonumber \\
&\leq& \sum_{A'\in \A_<'} \Delta_{\Y}^2\cdot7 \frac{\V + \ln(4/\delta')}{n} \nonumber \\
&\leq&  7\Delta_{\Y}^2\cdot\abs{\A'}\cdot\frac{\V + \ln(4/\delta')}{n}. \label{eq:bound1}
\end{eqnarray}
\par
For the integration over $\A_>'$, we first apply (\ref{eq:factor}), and recall lemmas \ref{lem:bias} and \ref{lem:var} to have that with probability at least $1-\delta'$ over $\Yspl$,
\begin{eqnarray}
&\displaystyle\sum_{A'\in \A_>'}&\int_{A'} \norm{f_{n, \A}(x) - f(x)}^2 \mu(dx)\nonumber\\
&=& \sum_{A'\in \A_>'}\int_{A'} \norm{f_{n, \A'}(x) - f(x)}^2 \mu(dx) \nonumber\\
&\leq& \sum_{A'\in \A_>'}2\lambda^2\Delta^2\paren{A'}\cdot \mu(A') \nonumber \\
&\,& + \sum_{A'\in \A_>'}2\Delta_{\Y}^2\cdot\frac{2 + \ln(\abs{\A'}/\delta')}{n \mu_n(A')}\cdot\mu(A')\nonumber\\
&\leq&  \sum_{A'\in \A_>'}2\lambda^2\Delta^2\paren{A'}\cdot 7\mu_n(A')\nonumber \\
&\,& +\sum_{A'\in \A_>'}2\Delta_{\Y}^2\cdot\frac{2 + \ln(\abs{\A'}/\delta')}{n \mu_n(A')}
\cdot7\mu_n(A')\nonumber\\
&\leq& 14\lambda^2\sum_{A'\in \A_>'}\mu_n(A')\Delta^2\paren{A'} \nonumber \\
&\,& + 14\Delta_{\Y}^2\abs{\A'}\cdot\frac{2 + \ln(\abs{\A'} /\delta')}{n}.
\label{eq:bound2}
\end{eqnarray}
\par
Note that the term $\ln \abs{\A'}$ in (\ref{eq:bound2}) is at most $O(\ln n)$ since the entire tree has height at most $6\log n$. Combining the bounds in (\ref{eq:bound1}) and (\ref{eq:bound2}), we get that there exists a constant $C_0$ such that $\norm{f_{n, \A} - f}^2$ is at most
\begin{eqnarray*}
&\,&C_0\left(\vphantom{\sum_{A'\in \A'}\mu_n(A')\Delta^2}\Delta_{\Y}^2\cdot\abs{\A}\frac{\log^2 n\log\!\log 1/\delta + \log (1/\delta')}{n}\right.\\
&\,&\left. +\lambda^2\sum_{A'\in \A'}\mu_n(A')\Delta^2\paren{A'}\right),
\end{eqnarray*}
with probability at least $1-2\delta'$. 
\par 
Setting $\delta' = \delta/36\log n$, the lemma follows by a union bound over at most $6\log n$ partitions in $\left\{\A^j\right\}_{j = 0}^i$, and then calling on lemma \ref{lem:properties_A'}.
\end{proof}

\section{Risk of final regressor $f_n \doteq f_{n,\Ar}$}
\label{sec:risk}
In this section we bound the excess risk of the final regressor $f_n \doteq f_{n,\Ar}$ in terms of the diameter decrease rate attained when $\mathtt{adaptiveRPtree}$ stops.
\par
To see that the stopping criteria eventually hold, note that the implementation of $\mathtt{coreRPtree}$ ensures that all cells at some level down the hierarchy have a single data point in them (see remark \ref{rem:treeSize}). In other words, we have $\diameter{\A^i} = 0$ eventually, forcing either stopping criterion to hold. 
\par 
We now outline the arguments in this section. For simplicity, assume $\Delta_\X$, $\Delta_\Y$, and $\lambda$ are all $1$. Consider some RPtree partition $\A$ and let $\diameter{\A} \approx \zeta$ for some scalar $\zeta$, we then have $\abs{\A}\lesssim \zeta^{-k}$ where $k$ is the diameter decrease rate attained by the algorithm. From lemma \ref{lem:risk} above, we roughly have 
$\norm{f_{n,\A} - f}^2 \lesssim \zeta^{-k}/n + \zeta^2,$ and the best bound is obtained  by setting $\zeta \approx n^{-1/(2+k)}$. Provided we pick an appropriate partition which optimizes $\zeta$, the final bound would then take the form $\norm{f_{n,\Ar} - f}^2 \lesssim n^{-2/(2+k)}$.

\subsection{Risk bound for cross-validation option}
\begin{lemma}[Existence of a good pruning]
Suppose the cross-validation option is used, and $\mathtt{adaptiveRPtree}$ attains a diameter decrease rate of $k$ on $\Xspl$. Define $$\alpha(n)\doteq\paren{\log^2 n}\log\log(n/\delta)+ \log (1/\delta),$$ and $\zeta \doteq \paren{\frac{\Delta_{\Y}^2\cdot\alpha(n)}{\lambda^2\diamX\cdot n}}^{1/(2+k)}$.
Let $n\geq \max \left\{\paren{\frac{\lambda\Delta_\X}{\Delta_\Y}}^2,
\alpha(n)\right\}$,
and for $i\geq 0$, let $\A^i$ as defined in $\mathtt{adaptiveRPtree}$. Then there exists $i_0\geq 0$ such that $\diameter{\A^{i_0}}\leq 2\zeta\cdot\diameter{\X}$ and $\abs{\A^{i_0}}\leq \zeta^{-k}$.
\label{lem:partition01}
\end{lemma}
\begin{proof}
Let $i\geq 0$. We have by definition that $\diameter{\A^i}\leq 2^{-i}\diameter{\X}$, while it follows from the assumption on diameter decrease rate that $\lev{\A^i}\leq ki$. Now let $\A^i$ be the last partition of $\X$ achieved by $\mathtt{adaptiveRPtree}$ when the stopping criteria holds. We have either that $\Delta_n(\A^i) = 0 < \zeta\cdot\diameter{\X}$, or 
$$ki\geq \lev{\A^i}\geq \log n^2 \geq k\log n^{2/(k+2)} \geq k \log 1/\zeta,$$ implying that
$\diameter{\A^i} \leq 2^{-i}\cdot\diameter{\X} \leq \zeta\cdot\diameter{\X}$.
\par 
Now, let $j\in{1, \ldots, i}$ be the first $j$ such that $\diameter{\A^j}\leq \zeta\cdot\diameter{\X}$. We consider the following two cases:
\begin{itemize}
 \item Either $\lev{\A^j}\leq \log \zeta^{-k}$, and we get $\abs{\A^j}\leq\zeta^{-k}$.
 \item Or $\lev{\A^j}>\log \zeta^{-k}$ in which case the following must hold:
\begin{itemize}
 \item $\diameter{\A^{j-1}}\leq 2\zeta\cdot\diameter{\X}$, since 
$kj\geq \lev{\A^j}\geq k \log1/\zeta$, implying that $j-1\geq\log(1/2\zeta)$.
\item $\lev{\A^{j-1}}< \log\zeta^{-k}$, for otherwise $j-1\geq \log 1/\zeta$ implying that $\diameter{\A^{j-1}}\leq \zeta\diameter{\X}$. It follows that $\abs{\A^{j-1}}\leq\zeta^{-k}$
\end{itemize}
\end{itemize}
Thus, either $\A^j$ or $\A^{j-1}$ satisfies the claim.
\end{proof}

\begin{lemma}
There exists a constant $C$ independent of $d$ and $\mu(\X)$, such that the following holds with probability at least $1- 2\delta/3$ over $(\Xspl, \Yspl)$ and the randomness in the algorithm. 
\par
Suppose the cross-validation option is used, and procedure $\mathtt{adaptiveRPtree}$ attains a diameter decrease rate of $k\geq d$ on $\Xspl$. Define $$\alpha(n)\doteq\paren{\log^2 n}\log\log(n/\delta)+ \log (1/\delta),$$ and assume $n\geq \max \left\{\paren{{\lambda\Delta_\X}/{\Delta_\Y}}^2, \alpha(n)\right\}$. The excess risk of the regressor is then bounded as
\begin{eqnarray*}
 \norm{f_n - f}^2 &\leq& C\cdot\paren{\lambda\Delta_\X}^{2k/(2+k)}\paren{\Delta_{\Y}^2\cdot\frac{\alpha(n)}{n}}^{2/(2+k)} \\
&\,& + 2\Delta_\Y^2\sqrt{\frac{\ln\log n^6 + \ln 3/\delta}{2n}}.
\end{eqnarray*}
\label{lem:riskCVoption}
\end{lemma}
\begin{proof}
Let $\A^{i_0}$ be as in lemma \ref{lem:partition01}, and $\zeta\doteq \paren{\frac{\Delta_{\Y}^2\cdot\alpha(n)}{\lambda^2\diamX\cdot n}}^{1/(2+k)}$. By applying lemma \ref{lem:risk} and then lemma \ref{lem:partition01}, we have with probability at least $1-\delta/3$ that  
\begin{eqnarray*}
\norm{f_{n,\A^{i_0}} -f}^2 &\leq& C_1\left(\Delta_{\Y}^2\abs{\A^{i_0}}
\frac{\alpha(n)}{n}\right.\\
&\,& \left. + \lambda^2\paren{\diam{\A^{i_0}} + n^{-4/(2+d)}\diamX}\vphantom{\frac{\alpha(n)}{n}}\right) \\
&\leq& C_1\paren{\Delta_{\Y}^2\cdot\zeta^{-k}
\frac{\alpha(n)}{n}
+ 5\lambda^2\zeta^2\diamX} \\
&\leq& C_2 \lambda^2\diamX\zeta^2.
\end{eqnarray*}
To analyze the cross validation phase, we first fix the partition tree and consider the obtained partitions from $\A^0$ to the final partition $\A^i$ when the stopping criteria holds. We have with probability at least $1-\delta/3$ over the choice of $(\Xspl', \Yspl')$ that $\forall j \in \{0, \ldots, i\}$
\begin{equation*}
 \abs{R\paren{f_{n, \A^j}} - R_n'\paren{f_{n, \A^j}}}
\leq \Delta_\Y^2\sqrt{\frac{\ln\log n^6 + \ln 3/\delta}{2n}}.
\end{equation*}
The above is obtained by applying McDiarmid's to the empirical risk followed by a union bound over at most $6\log n$ regressors $f_{n, \A^j},j \in \{0, \ldots, i\}$.
\par 
Let $f_n \doteq f_{n,\Ar}$ be the empirical risk minimizer, we can then conclude that 
$$\norm{f_{n} - f}^2\leq C_2 \lambda^2\diamX\zeta^2 + 2\Delta_\Y^2\sqrt{\frac{\ln\log n^6 + \ln 3/\delta}{2n}}$$ with probability at least $1-2\delta/3$. 
\end{proof} 

\subsection{Risk bound for automatic stopping option}
 
\begin{lemma}[Properties of $\Ar$]
\label{lem:partition02}
Suppose the automatic stopping option is used, and that $\mathtt{adaptiveRPtree}$ attains a diameter decrease rate of $k$ on $\Xspl$. Define $$\alpha(n)\doteq\paren{\log^2 n}\log\log(n/\delta)+ \log (1/\delta),$$ and $\zeta \doteq \paren{\frac{\alpha(n)}{n}}^{1/(2+k)}$. Finally, assume $n \geq \alpha(n)$. Then, the following holds for the final partition $\Ar$ retained for regression: 
\begin{eqnarray*}
\paren{\frac{\alpha(n)}{n}\cdot\abs{\Ar} + \diam{\Ar}} \leq \paren{4\diam{\X} + 1}\zeta^2.
\end{eqnarray*}
\end{lemma}
\begin{proof}
For $i\geq 0$, let $\A^i$ as defined in $\mathtt{adaptiveRPtree}$. We have by definition that $\diameter{\A^i}\leq 2^{-i}\diameter{\X}$, while it follows from the assumption on diameter decrease rate that $\lev{\A^i}\leq ki$. Now for some $i\geq 1$, let $\A^i$ be the final partition of $\X$ achieved by $\mathtt{adaptiveRPtree}$ when the stopping criteria holds.
We consider the following two cases:
\begin{itemize}
 \item Either $\lev{\A^i}\leq \log\zeta^{-k}$, and we have by the stopping condition that:
\begin{eqnarray*}
\diam{\A^i}&\leq& \frac{\alpha(n)}{n}2^{\lev{\A^i}}\cdot\diam{\X} \\
&\leq& \frac{\alpha(n)}{n}\zeta^{-k}\diam{\X} = \zeta^2\diam{\X}.
\end{eqnarray*}
\item Or $\lev{\A^i}> \log{\zeta^{-k}}$, in which case the following must hold:
\begin{itemize}
\item $\diameter{\A^{i-1}} \leq 2\zeta\cdot\diameter{\X}$, since 
$ki\geq \lev{\A^i}\geq k \log(1/\zeta)$, implying that $i-1\geq\log(1/2\zeta)$.
\item $\lev{\A^{i-1}}< \log\zeta^{-k}$, for otherwise we would have stopped at $i-1$. To see this, assume instead that $\lev{\A^{i-1}} \geq \log\zeta^{-k}$: we have  
that $(i-1) \geq \log\frac{1}{\zeta}$ and subsequently that 
\begin{eqnarray*}
 \diam{\A^{i-1}}&\leq& 2^{-2(i-1)}\diam{\X} \leq \zeta^2\diam{\X} \\
&=& \frac{\alpha(n)}{n}\cdot\zeta^{-k}\cdot\diam{\X}\\
&\leq&  \frac{\alpha(n)}{n}2^{\lev{\A^{i-1}}}\cdot\diam{\X}.
\end{eqnarray*}
In other words, $$\lev{\A^{i-1}} \geq \log\left(n\diam{\A^{i-1}}/\alpha(n)\diam{\X}\right).$$
\end{itemize}

\end{itemize}

In either case at least one of $\A^i$ and $\A^{i-1}$ has size at most $\zeta^{-k}$ and diameter at most $2\zeta\cdot\Delta_{\X}$. It follows that
\begin{eqnarray*}
&\,& \min_{j \in \left\{i-1,\, i\right\}} \paren{\frac{\alpha(n)}{n}\cdot \abs{\A^j} + \diam{\A^j}} \leq \\
&\,& \frac{\alpha(n)}{n}\cdot \zeta^{-k} 
 + 4\zeta^2\cdot\diam{\X} 
= \paren{4\diam{\X} + 1}\zeta^2,
\end{eqnarray*}
which concludes the argument.
\end{proof}

\begin{lemma}
 There exists a constant $C$ independent of $d$ and $\mu(\X)$, such that the following holds with probability at least $1- \delta/3$ over $(\Xspl, \Yspl)$ and the randomness in the algorithm. 
\par
Suppose the automatic stopping option is used; assume $\mathtt{adaptiveRPtree}$ attains a diameter decrease rate of $k\geq d$ on $\Xspl$. Define $\alpha(n)\doteq\paren{\log^2 n}\log\log(n/\delta)+ \log (1/\delta)$. The excess risk of the regressor is then bounded as
\begin{eqnarray*}
 \norm{f_n - f}^2 \leq C\cdot\paren{\Delta_{\Y}^2 + \lambda^2}(\diamX +1)\cdot\paren{\frac{\alpha(n)}{n}}^{2/(2+k)}.
\end{eqnarray*}
\label{lem:riskASoption}
\end{lemma}
\begin{proof}
For $n \leq \alpha(n)$, the bound on the excess risk holds vacuously. We assume henceforth that $n > \alpha(n)$. Let $\zeta \doteq \paren{\frac{\alpha(n)}{n}}^{1/(2+k)}$. By first applying lemma \ref{lem:risk} then lemma \ref{lem:partition02}, we have with probability at least $1-\delta/3$ that 
\begin{eqnarray*}
\norm{f_{n,\Ar} -f}^2 &\leq& C_1\left(\Delta_{\Y}^2\abs{\Ar}
\frac{\alpha(n)}{n} \right.\\
&\,&  \left. + \lambda^2\paren{\diam{\Ar} + n^{-4/(2+d)}\diamX}\vphantom{\frac{\alpha(n)}{n}}\right)\\
&\leq& C_1\paren{\Delta_{\Y}^2 + \lambda^2} \left(\abs{\Ar}
\frac{\alpha(n)}{n}\right.\\
&\,&  \left. +\paren{\diam{\Ar} + n^{-4/(2+d)}\diamX}\vphantom{\frac{\alpha(n)}{n}}\right)\\
&\leq& C_1\paren{\Delta_{\Y}^2 + \lambda^2}\paren{\paren{4\diamX +1}\zeta^2 + \zeta^2\diamX}\\
&\leq& C\paren{\Delta_{\Y}^2 + \lambda^2}\paren{\diamX +1}\zeta^2, 
\end{eqnarray*}
which concludes the argument.
\end{proof}

\section{Core RPtree and diameter decrease rates}
\label{sec:decreaserates}

\subsection{Core RPTree procedures}
\label{sec:coreRPtree}
\begin{procedure}[H]
$\A_0\leftarrow \{A_0\}$\;
\For{$i\leftarrow 1$ \KwTo $\infty$}{
\If{$\diameter{\A_{i-1}} \leq \Delta$}{
\Return\;
}
Choose a random direction $v \sim \mathcal{N}\paren{0, \frac{1}{D}I_D}$\;
Choose a random $\tau\sim \mathcal{U}[-1, 1]\cdot \frac{6}{\sqrt{D}} \Delta_n(A_0)$\;
\vspace{1mm}
  \ForEach{cell $A\in\A_{i-1}$} {
\eIf{$(l+i)$ is odd} {
\CommentSty{// Noisy splits.}\\
$t \leftarrow \text{median}\{z^\Tr v: z \in \Xspl\cap A_0\} + \tau$\;
}
{
\CommentSty{// Median splits.}\\
$t\leftarrow\text{median}\{z^\Tr v: z \in \Xspl\cap A\}$\;
}
$A_{\text{left}} \leftarrow \{x\in A,\, x^\Tr v\leq t\}$\;
$A_{\text{right}} \leftarrow A\setminus A_{\text{left}}$\;
}
$\A_i\leftarrow$ partition of $A_0$ defined by the leaves of the current tree\;
}
\caption{basicRPtree($A_0\subset\X$, $\Delta$, level $l$)\label{proc:1}}
\end{procedure}

\begin{procedure}[H]
\SetKwFunction{rptree}{basicRPtree}
 Call \rptree{$A_0, \Delta, l$} $\log\paren{{6n^2}/{\delta}}$ times and return the shortest tree.
 \caption{coreRPtree($A_0\subset\X$, $\Delta$, $\delta$, level $l$)}
\end{procedure}


RPtree consists of hierarchically bisecting the data space with random hyperplanes. In $\mathtt{basicRPtree}$ we alternate between two types of bisections: we split exactly at the median in order to balance the tree, while we split at the median + noise to improve the rate at which the data diameters are reduced down the tree. Notice that for the ``noisy'' split we use the same hyperplane to bisect all nodes $A\in \A_{i-1}$.
\par 
The procedure $\mathtt{coreRPtree}$ serves to boost the probability that we get a small tree. The many calls to $\mathtt{basicRPtree}$ can be done in parallel so that we don't keep growing the trees that are to be discarded once the smallest tree is identified.

\begin{remark}
\label{rem:treeSize}
Given the implementation of $\mathtt{coreRPtree}$, the tree returned by $\mathtt{adaptiveRPtree}$ has the following properties:
\begin{itemize}
 \item Any node at level $6\log n$ has at most $1$ data point: the data is split at the exact median at every other level so that the number of points per nodes decreases exponentially from the root down. If $n$ were a power of $2$, we'd need at most $2\log n$ levels to get to $1$ point per node. For general $n$, notice that the number of points in a node at level $i\geq 2$ is at most $\frac{3}{4}$ of that of its ancestor at level $i-2$. In other words we need at most $2\log n/\log(4/3)\leq 6\log n$ levels to get down to $1$ point per node.
\item As a consequence, the entire tree reaches depth at most $6\log n$ under either stopping criteria, and therefore has at most $2n^6$ nodes.
\item Another consequence is that at most $2n^6\log (6n^2/\delta)$ random directions are required to build the entire tree.
\end{itemize}

\end{remark}

\subsection{Worst case decrease rates}
In this section we consider worst case bounds for the diameter decrease rates attainable by the algorithm over supports of low intrinsic dimension.
\par
The following theorem, adapted from Dasgupta and Freund \cite{DF:59}, is the core of the argument.

\begin{theorem}
 \label{theo:rptree}
Let $A\subset \real^D$ and suppose $A\cap\Xspl$ has Assouad dimension $d$. There exists a constant $C'$ independent of the sample $\Xspl$ and $d$, with the following property. We have with probability at least $\frac{1}{2}$ that the tree rooted at $A$ returned by the call $\mathtt{basicRPtree}(A, \diameter{A}/2, l)$ has depth at most $C' d\log d$.
\end{theorem}
\begin{proofIdea}
The proof is a direct consequence of lemma 9 of \cite{DF:59} applied to the ``noisy'' splits at alternating levels in procedure $\mathtt{basicRPtree}$.
\par
Let $r = \Delta_n(A)/512\sqrt{d}$ and consider an $r$-cover of $A$; now consider pairs of balls $B = B(z, r)$, $B'=B(z', r)$, where $z, z'$ are in the cover and $\norm{z-z'} \geq \frac{1}{2}\Delta_n(A)-2r$. Notice that $\mathtt{basicRPtree}$ stops if for all such pairs, no leaf of the tree contains points from both $B\cap\Xspl$ and $B'\cap\Xspl$.
\par 
Fix such a pair $B$ and $B'$. By lemma 9 of \cite{DF:59}, every ``noisy'' split has a constant probability of separating $B\cap\Xspl$ and $B'\cap\Xspl$. Thus, the probability that some cell at level $i$ contains points from both $B\cap\Xspl$ and $B'\cap\Xspl$ goes down exponentially with $i$. A union bound over at most $(O(d)^d)$ such pairs yields the theorem.
\end{proofIdea}

\begin{corollary}
\label{cor:diameter-decrease-max}
Suppose $\X$ has Assouad dimension $d$. Let $C'$ be as in theorem \ref{theo:rptree}. Fix $\Xspl$. With probability at least $1-\delta/3$ over the randomness in the algorithm, the procedure $\mathtt{adaptiveRPtree}$ attains a diameter decrease rate of $k \leq C' d\log d$ on $\Xspl$.
\end{corollary}
\begin{proof}
Consider a subtree rooted at $A$ returned by the call $\mathtt{coreRPtree}(A, \diameter{A}/2, \delta, l)$ in the second loop of procedure $\mathtt{adaptiveRPtree}$. Since $\X$ has Assouad dimension $d$, $A\cap\Xspl$ also has Assouad dimension $d$ by definition so theorem \ref{theo:rptree} holds.
\par 
Procedure $\mathtt{coreRPtree}$ calls $\mathtt{basicRPtree}$ as many as $\log \paren{6n^2/\delta}$ times and returns the smallest tree; thus the probability that the subtree rooted at $A$ has depth over $C' d\log d$ is at most $\delta/6n^2$. Now, under both stopping conditions, $\mathtt{coreRPtree}$ is only called on nodes at level at most $\log n^2$; a union bound over all such nodes (at most $2n^2$) yield a probability of failure at most $\delta/3$. 
\end{proof}

\section{Final Remarks}
We have shown in this paper that an RPtree regressor will perform well in a scenario where the data space $\X$ has low Assouad dimension $d<\!\!<D$. 
\par 
Our results are easily extended to other settings. We can for example consider a scenario where the data has low Assouad dimension $d$ at small resolution but ``fills'' up space at higher resolution. One may think for instance of a Hilbert space filling curve where balls of small enough radius have low Assouad dimension relative to the entire space.
 (see figure \ref{fig:spaceFilling}).
\begin{figure}
\centering
\includegraphics[height=4cm]{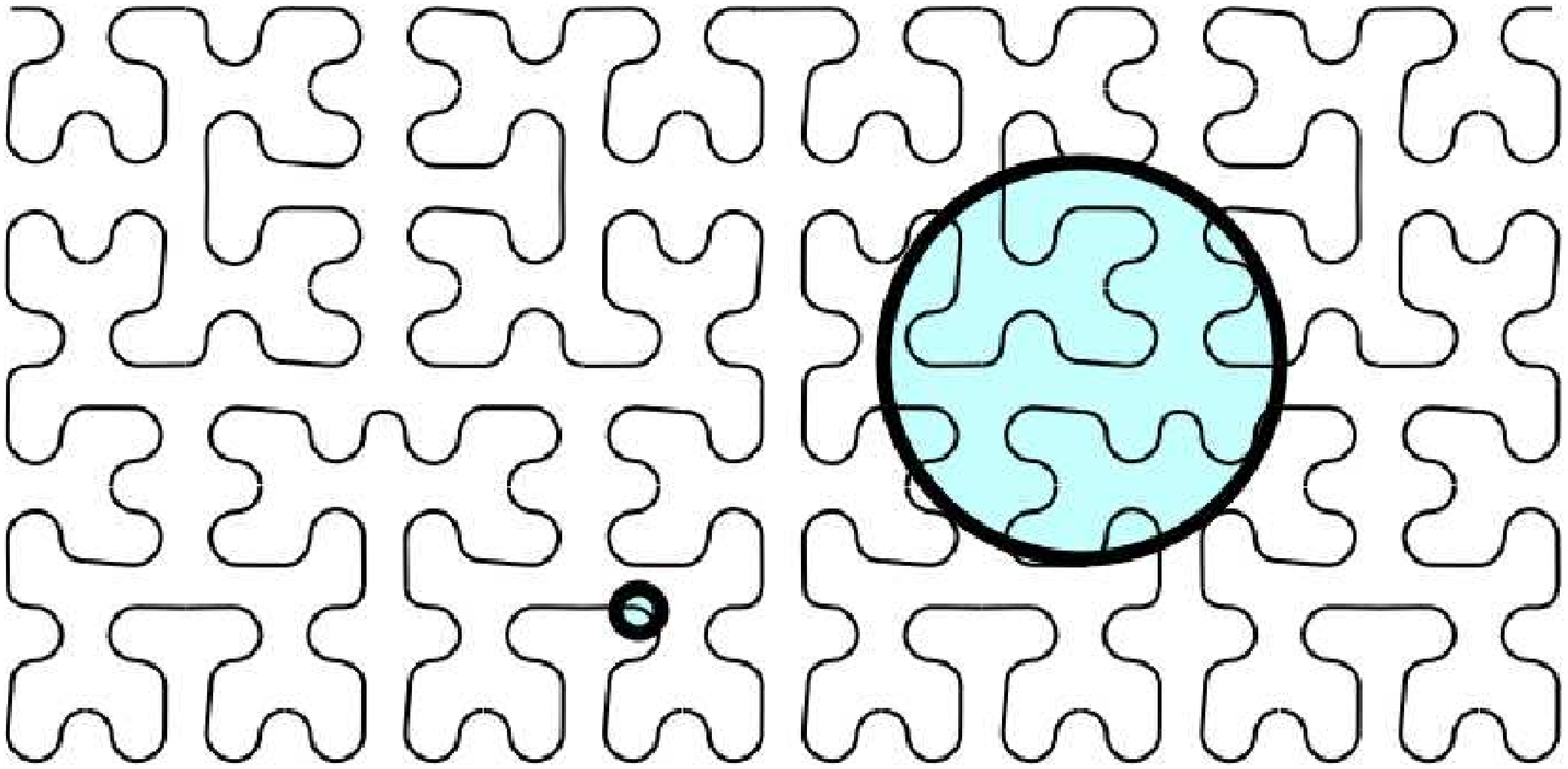} 
\caption{Hilbert space filling curve, balls of smaller radius have lower Assouad dimension.}
\label{fig:spaceFilling}
\end{figure}
RPtree in this case would initially decrease diameter at a slow rate till it arrives at small enough neighborhoods, at which time the diameter decrease rates speed up. Even in this case, the complexity of the data in larger regions of space has little effect on the final excess risk, provided $n$ is large enough for the tree to arrive at well populated regions with sufficiently small diameter. 

%
\bibliography{refs}

\begin{thebibliography}{BKL06}

\bibitem[BKL06]{BKL:71}
A~Beygelzimer, S.~Kakade, and J.~Langford.
\newblock Cover trees for nearest neighbors.
\newblock {\em ICML}, 2006.

\bibitem[BL06]{BL:65}
P.~Bickel and B.~Li.
\newblock Local polynomial regression on unknown manifolds.
\newblock {\em Tech. Re. Dep. of Stats. UC Berkley}, 2006.

\bibitem[BN03]{BN:63}
M.~Belkin and N.~Niyogi.
\newblock Laplacian eigenmaps for dimensionality reduction and data
  representation.
\newblock {\em Neural Computation}, 15:1373--1396, 2003.

\bibitem[Cla05]{C:74}
K.~Clarkson.
\newblock Nearest-neighbor searching and metric space dimensions.
\newblock {\em Nearest-Neighbor Methods for Learning and Vision: Theory and
  Practice}, 2005.

\bibitem[DF08]{DF:59}
S.~Dasgupta and Y.~Freund.
\newblock Random projection trees and low dimensional manifolds.
\newblock {\em STOC}, 2008.

\bibitem[GLZ08]{GLZ:73}
A.~B. Goldberg, M.~Li, and X.~Zhu.
\newblock Online manifold regularization: a new learning setting and empirical
  study.
\newblock {\em ECML PKDD}, 2008.

\bibitem[GN05]{GN:67}
S.~Gey and E.~Nedelec.
\newblock Model selection for cart regression trees.
\newblock {\em IEEE Transactions on Information Theory}, 51, 2005.

\bibitem[IN07]{IN:70}
P.~Indyk and A.~Naor.
\newblock Nearest neighbor preserving embedding.
\newblock {\em ACM Transactions on Algorithms}, 2007.

\bibitem[LGL96]{DGL:73}
L.Devroye, L.~Gyorfi, and G.~Lugosi.
\newblock {\em A Probabilistic Theory of Pattern Recognition}.
\newblock Springer, 1996.

\bibitem[LW07]{LW:68}
J.~Lafferty and L.~Wasserman.
\newblock Statistical analysis of semi-supervised regression.
\newblock {\em NIPS}, 2007.

\bibitem[RS00]{RS:62}
S.~Roweis and L.~Saul.
\newblock Nonlinear dimensionality reduction by locally linear embedding.
\newblock {\em Science}, 290:2323--2326, 2000.

\bibitem[SN06]{SN:66}
C.~Scott and R.D. Nowak.
\newblock Minimax-optimal classification with dyadic decision trees.
\newblock {\em IEEE Transactions on Information Theory}, 52, 2006.

\bibitem[Sto80]{S:60}
C.~J. Stone.
\newblock Optimal rates of convergence for non-parametric estimators.
\newblock {\em Ann. Statist.}, 8:1348--1360, 1980.

\bibitem[Sto82]{S:61}
C.~J. Stone.
\newblock Optimal global rates of convergence for non-parametric estimators.
\newblock {\em Ann. Statist.}, 10:1340--1353, 1982.

\bibitem[TSL00]{TDL:64}
J.B. TenenBaum, V.~De Silva, and J.~Langford.
\newblock A global geometric framework for non-linear dimensionality reduction.
\newblock {\em Science}, 290:2319--2323, 2000.

\bibitem[VC71]{VC:72}
V.~Vapnik and A.~Chervonenkis.
\newblock On the uniform convergence of relative frequencies of events to their
  expectation.
\newblock {\em Theory of probability and its applications}, 16:264--280, 1971.

\end{thebibliography}
%

\end{document}